\documentclass{article}

\usepackage[preprint]{neurips2024}

\usepackage[utf8]{inputenc} 
\usepackage[T1]{fontenc}    
\usepackage{hyperref}       
\usepackage{url}            
\usepackage{booktabs}       
\usepackage{amsfonts}       
\usepackage{nicefrac}       
\usepackage{microtype}      
\usepackage{xcolor}         
\usepackage{amsmath}
\usepackage{amssymb}
\usepackage{amsthm}
\usepackage{graphicx}
\usepackage{wrapfig}
\usepackage{comment}
\usepackage{enumitem}

\newtheorem{theorem}{Theorem} 
\newtheorem{lemma}[theorem]{Lemma}
\newtheorem{proposition}[theorem]{Proposition} 
\newtheorem{definition}[theorem]{Definition}

\newcommand{\LR}{\mathbf{LR}}

\newcommand{\R}{\mathbb{R}}

\newcommand{\N}{\mathcal{N}}

\newcommand{\Exp}{\mathop\mathbb{E}}

\newcommand{\rank}{\text{rank}}
\newcommand{\ravid}[2]{\textcolor{blue}{Ravid: #1}}

\usepackage{thmtools}
\usepackage{thm-restate}

\title{Learning to Compress: Local Rank and \\ Information Compression in Deep Neural Networks}

\author{%
  Niket Patel\\
  Department of Mathematics\\
  University of California, Los Angeles\\
  \texttt{niketpatel@ucla.edu} \\
  \And
  Ravid Shwartz Ziv\\
  New York University, Wand.AI\\
  \texttt{ravid.shwartz.ziv@nyu.edu} \\
}

\begin{document}

\maketitle

\begin{abstract}
Deep neural networks tend to exhibit a bias toward low-rank solutions during training, implicitly learning low-dimensional feature representations. This paper investigates how deep multilayer perceptrons (MLPs) encode these feature manifolds and connects this behavior to the Information Bottleneck (IB) theory. We introduce the concept of \emph{local rank} as a measure of feature manifold dimensionality and demonstrate, both theoretically and empirically, that this rank decreases during the final phase of training. We argue that networks that reduce the rank of their learned representations also compress mutual information between inputs and intermediate layers. This work bridges the gap between feature manifold rank and information compression, offering new insights into the interplay between information bottlenecks and representation learning.
\end{abstract}

\section{Introduction}

The \textbf{Data Manifold Hypothesis} \citep{olah2014neural} suggests that real-world datasets lie on a manifold whose intrinsic dimensionality is much lower than the ambient input space. While numerous models can fit training data, those that generalize well and exhibit robustness likely learn meaningful representations of this underlying manifold. Recent studies have observed that deep neural networks, particularly multilayer perceptrons (MLPs), exhibit an emergent bottleneck structure, where certain layers effectively compress input data into lower-dimensional feature manifolds \cite{jacot2}.

Understanding how this emergent structure aligns with existing theories, such as the Information Bottleneck (IB) theory \citep{tishby2000information, shwartz2017opening}, is crucial to advance our understanding of representation learning in deep networks. This paper aims to provide a theoretical and empirical exploration of this phenomenon by introducing the concept of \emph{local rank} and analyzing its relationship with information compression during training.

\paragraph{This paper makes the following key contributions:}

\begin{itemize}
    \item \textbf{Definition and Analysis of Local Rank:} We define \emph{local rank} as a metric to quantify the dimensionality of feature manifolds within a neural network. We provide theoretical insights into how local rank behaves during training and establish bounds based on implicit regularization.

    \item \textbf{Empirical Evidence of Rank Reduction:} We conduct experiments on synthetic and real-world datasets to demonstrate that local rank decreases during the terminal phase of training, indicating that networks compress the dimensionality of their learned representations.

    \item \textbf{Connection to Information Bottleneck Theory:} We explore the relationship between local rank and information compression, arguing that a reduction in local rank correlates with mutual information compression between inputs and intermediate representations.
\end{itemize}

The remainder of this paper is organized as follows: Section \ref{sec:related} provides a comprehensive review of related literature, situating our work within the broader context of implicit regularization, low-rank bias, and the Information Bottleneck theory. Section \ref{sec:notation} introduces the notation used throughout the paper. In Section \ref{sec:LR}, we define the local rank and present theoretical and empirical analyses. Section \ref{sec:IB} discusses the information-theoretic implications of local rank, connecting it to the Information Bottleneck theory. Finally, Section \ref{sec:discussion} concludes the paper and outlines future research directions.

\section{Related Work}
\label{sec:related}

\paragraph{Implicit Regularization in Deep Learning.} Implicit regularization refers to the phenomenon in which the training dynamics of neural networks, particularly under gradient descent, lead to solutions with desirable properties without explicit regularization terms. \citet{neyshabur2014search} and \citet{zhang2021understanding} investigated how overparameterized networks generalize despite being able to fit random labels. \citet{gunasekar2017implicit} and \citet{arora2019implicit} studied implicit bias in linear and deep matrix factorization, showing that gradient descent favors low-rank solutions.

\paragraph{Low-Rank Representations and Manifolds.} The concept of neural networks learning low-dimensional manifolds has been explored in various contexts. \citet{papyan2020prevalence} introduced the notion of \emph{neural collapse}, where class means converge to a simplex Equiangular Tight Frame at the final layer. \citet{ansuini2019intrinsic} and \citet{ben2023reverse} measured intrinsic dimensionality in deep networks, observing that representations become more compressed in deeper layers. \citet{súkeník2024neuralcollapseversuslowrank} explicitly analyzes the relation between low-rank bias and neural collapse in deep networks.

\paragraph{Rank of Jacobians and Feature Maps.} Closely related to the local rank, \citet{jacot2023implicit} and \citet{jacot2} analyzed the rank of Jacobians in neural networks, connecting it to generalization properties. \citet{humayun2024understanding} studies a similar object in the context of assessing the geometry of diffusion models.\citet{feng2022rank} studied the low-rank structure in the Jacobian matrices of neural networks, showing that rank deficiency can lead to better generalization.

\paragraph{Information Bottleneck Theory.} The Information Bottleneck (IB) framework \citep{tishby2000information} provides a theoretical lens to understand how neural networks balance compression and prediction. \citet{shwartz2017opening} applied IB to deep learning, proposing that networks first fit the data and then compress the representations. Subsequent works, such as \citet{michael2018on} and \citet{NEURIPS2023_6b1d4c03}, explored the universality of this phenomenon, leading to a richer understanding of information dynamics in training.

\paragraph{Relation to Our Work.} Our paper builds on these foundational studies by leveraging a measurable quantity—the local rank—that captures the dimensionality of the learned feature manifolds. We theoretically link this to the implicit regularization of the ranks of weight matrices and empirically demonstrate its reduction during training. Moreover, we connect these geometric properties of neural representations to information-theoretic principles, offering new insights into how networks compress information.

\section{Notation}
\label{sec:notation}
We consider a neural network $\mathcal{N}$ parameterized by $\theta$, mapping inputs to outputs as $\mathcal{N}_\theta: \R^{n_0} \to \R^{n_L}$. The network has depth $L$, and each layer $l \in \{1, \dots, L\}$ consists of an affine transformation followed by a nonlinearity:
\begin{equation}
    h_l(x) = \phi(A_l(h_{l-1}(x))) = \phi(W_l h_{l-1}(x) + b_l),
\end{equation}
where $h_0(x) = x$, and $h_l (x) = \N_\theta(x)$, $W_l \in \R^{n_l \times n_{l-1}}$ is the weight matrix, $b_l \in \R^{n_l}$ is the bias vector, and $\phi(\cdot)$ is an element-wise activation function (e.g., ReLU). The pre-activation at layer $l$ is $p_l(x) = W_l h_{l-1}(x) + b_l$.

The Jacobian matrix of a function $f: \R^n \to \R^m$ at point $x$ is denoted by $J_x f \in \R^{m \times n}$. The rank of a matrix $A$ is denoted by $\rank(A)$, and the $\epsilon$-rank, $\rank_\epsilon(A)$, counts the number of singular values of $A$ greater than $\epsilon$.

\section{The Local Rank of Representations}
\label{sec:LR}

In this section, we introduce the concept of \emph{local rank} as a measure of the dimensionality of feature manifolds learned by neural networks. Consider the data distribution with support $\Omega \subseteq \R^{n_0}$. The feature manifold at layer $l$ is defined as $\mathcal{M}_l = p_l(\Omega)$, where $p_l$ maps input data to pre-activations at layer $l$.

\begin{definition}
The \textbf{local rank} at layer $l$, denoted as $\LR_l$, is defined as the expected rank of the Jacobian of $p_l$ with respect to the input:
\begin{equation}
    \LR_l = \Exp_{x \sim \text{Data}} \left[ \rank(J_x p_l) \right].
\end{equation}

Since the set of matrices without full rank is measure 0, we find it more convenient to look at an approximation for the local rank. For a threshold $\epsilon > 0$, the \textbf{robust local rank} at layer $l$ is:
\begin{equation}
    \LR_l^\epsilon = \Exp_{x \sim \text{Data}} \left[ \rank_\epsilon(J_x p_l) \right].
\end{equation}

\end{definition}

This is a meaningful measure of the rank of the feature manifold since the Jacobian's null space identifies the input dimensions which vanish near $x$, so its rank captures the true number of feature dimensions.

\subsection{Theoretical Analysis of Local Rank}

We investigate how the local rank behaves under gradient flow dynamics and its connection to implicit regularization. Under suitable assumptions \citep{wang2021implicit, Lyu2020Gradient, NEURIPS2020_c76e4b2f}, solutions to the gradient flow ODE with exponential tailed losses have been shown to converge in direction to a Karush-Kuhn-Tucker (KKT) point of the following optimization problem, where $\{(x_i, y_i)\}_{i=1}^n \subseteq \R^{n_0} \times \{-1, 1\}$ is the training dataset:
\begin{equation}
\label{optimization_problem}
\min_\theta \frac{1}{2} \|\theta\|^2 \quad \text{subject to} \quad \forall i \in [n], \, y_i \mathcal{N}_\theta(x_i) \geq 1.
\end{equation}

This convergence to a KKT point implies that the solution minimizes the norm of the weights while satisfying the classification constraints. The minimization of the norm leads to implicit regularization effects, including the potential reduction in the rank of weight matrices.

If an intermediate layer exists with a low local rank, it can be viewed as a bottleneck in terms of the dimensionality of the feature manifold. We can establish a connection between Equation (\ref{optimization_problem}) and the existence of a bottleneck layer with low local rank. Specifically, we can prove the existence of such a bottleneck layer under the global optimum of this problem as a consequence of the implicit regularization of the ranks of weight matrices.

\begin{proposition}
\label{prop:local_rank_bound}
\textbf{(Informal)} Let $\mathcal D = \{(x_i, y_i)\}_{i=1}^n \subseteq \R^{n_0} \times \{-1, 1\}$ be a binary classification dataset. 
Assume there exists a fully connected neural network with weight matrices uniformly bounded by $B$ that correctly classifies every data point in $\mathcal D$ with margin 1.
Then, for $\theta = (W_1, \cdots W_L)$ parameterizing a $L$ layer neural network at the global optimum to \ref{optimization_problem}, there exists a layer $l$ and $\epsilon_0 > 0$ such that for all $0 < \epsilon < \epsilon_0$:
\begin{equation}
    \LR_l^\epsilon \leq \frac{2}{\epsilon^2} \left(\frac{B}{\sqrt{2}}\right)^{\frac{2K}{L}} \frac{L+1}{L} \|W_l\|_\sigma^2,
\end{equation}
where $\|W_l\|_\sigma$ denotes the operator norm of $W_l$.
\end{proposition}

\begin{proof}
The proof leverages recent results from implicit regularization \citep{timor2023implicit} and properties of the Jacobian of ReLU networks. A formal statement with definitions for all constants and proof are provided in Appendix \ref{sec:proofsclass}.
\end{proof}

We note that the right-hand side converges to $\frac{2\|W_l\|_\sigma^2}{\epsilon^2}$ as $L\to\infty$, so this bound is typically much better than the trivial bound on the local rank at layer $l$ given by $\textbf{LR}_l^\epsilon \leq ||W_l||_F^2 \slash \epsilon^2$. 
This proposition implies that during training, certain layers in the network develop low-rank weight matrices, leading to a reduced local rank in the representations. 

Next, we show an analogous result with even a tighter bound for minimum norm solutions to interpolating neural networks for training on \textbf{regression tasks}:

\begin{proposition} \textbf{(Informal)}
    Let $\{(x_i, y_i)\}_{i=1}^n \subset \mathbb{R}^{n_0} \times \mathbb{R}_+$ be a regression dataset. Assume there exists a fully connected neural network $\N$ with weight matrices uniformly bounded by $B$ that interpolates all data points, $\N(x_i) = y_i$. 
    Let $\N_\theta$ be a fully-connected neural network of depth $L$ parameterized by $\theta = [W_1, \dots, W_{L}] $, where $\theta$ is a global optimum of the following optimization problem:

\begin{equation}
\min_\theta \|\theta\| \quad \text{s.t.} \quad \forall i \in [n], \ \N_\theta(x_i) = y_i.
\end{equation}

Then, there exist an $l\in \{1, \cdots , L\}$ and $\epsilon_0 > 0$ such that for $0 < \epsilon <\epsilon_0$ the following holds:

\begin{equation}
   \textbf{LR}_l^\epsilon \leq \frac{||W_l||_\sigma^2 \cdot B^\frac{2K}{L}}{\epsilon^2}
\end{equation}

where $||W_l||_\sigma$ denotes the operator norm of $W_l$.
\end{proposition}

\begin{proof}
    As before, we leverage results on implicit regularization from \cite{timor2023implicit}, and a full proof and formal statement can be found in Appendix \ref{sec:proofsinterp}.
\end{proof}

\section{Empirical Evidence of Local Rank Reduction}

We empirically validate the theoretical insights by training MLPs on synthetic and real datasets and measuring the local rank during training.

\paragraph{Synthetic Data.} We train a 3-layer MLP with an input dimension of 100, 200 neurons per hidden layer, and an output dimension of 2. The inputs and outputs are Gaussian with a random cross-covariance matrix, reflecting a scenario where the network learns to map between correlated Gaussian distributions. Here, we use Adam Optimizer with a learning rate of $1e^{-4}$, with the Mean Squared Error Loss function.

\paragraph{MNIST Data.} We also train a 4-layer MLP on the MNIST dataset \citep{mnist}. Each hidden layer has 200 neurons. We use the Adam optimizer with a learning rate of $1e^{-4}$, with the Cross Entropy Loss function.

As shown in Figure~\ref{fig:local_rank_reduction}, we observe a significant drop in local rank during the final stages of training in both cases. This phenomenon occurs across all layers of the network, suggesting that neural networks inherently compress the dimensionality of their learned representations as they converge. This compression occurs in two phases, where the second phase corresponds to the compression phase identified by \citet{shwartzrepresentation} in the IB theory.

\begin{figure}[ht]
    \centering
    \includegraphics[width=\linewidth]{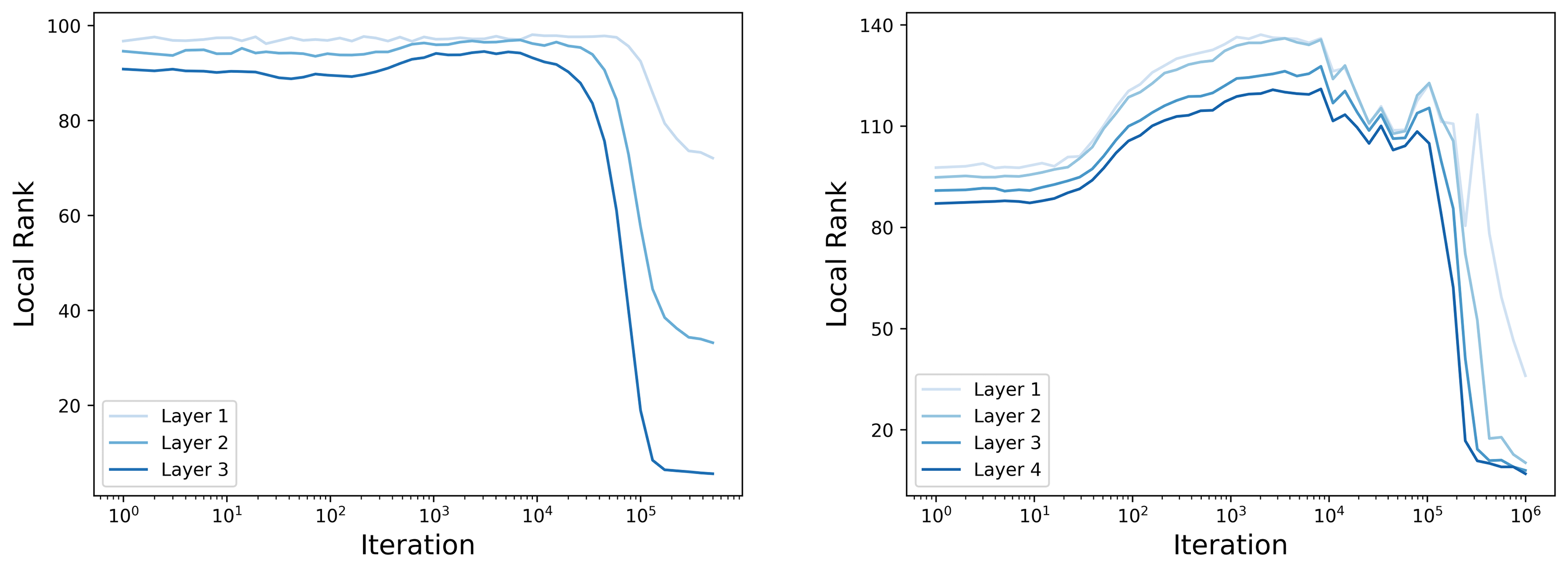}
    \caption{\textbf{Reduction in Local Rank During Training.}
    \textbf{Left:} A 3-layer MLP trained on synthetic Gaussian data.
    \textbf{Right:} A 4-layer MLP trained on MNIST. In both cases, the local rank decreases during the terminal phase of training, indicating compression of the feature manifold across all layers.}
    \label{fig:local_rank_reduction}
\end{figure}

\section{Information Theoretic Implications of Local Rank}
\label{sec:IB}

A core challenge in developing an understanding of deep learning is to define what constitutes as good representation learning. 
The Information Bottleneck (IB) framework provides a theoretical foundation for this by proposing that optimal representations are those which are maximally informative about the targets while compressing redundant information. 
In this section, we explore the relationship between local rank and information compression, positioning our findings within the context of IB theory to provide insight into the structure of learned representations.

\subsection{Information Bottleneck Framework}
The Information Bottleneck \citep{tishby2000information} and \citep{shwartz2022information} provides a principled approach to balance compression and relevance in representations. Given input $X$ and output $Y$, the goal is to find a representation $T$ that maximizes mutual information with $Y$ while minimizing mutual information with $X$. The IB Lagrangian is:
\begin{equation}
    \mathcal{L}_{\text{IB}} = I(T; X) - \beta I(T; Y),
\end{equation}
where $\beta$ controls the trade-off between compression and prediction.

\subsection{Local Rank and the Gaussian Information Bottleneck}

In general, finding analytical solutions to the IB problem is challenging. However, for jointly Gaussian variables, \citet{NIPS2003_7e05d6f8} found an explicit solution for the IB problem, which we can adjust:

\begin{theorem}
\label{thm:gaussian_ib}
\textbf{(Adapted from \citet{NIPS2003_7e05d6f8})} For jointly Gaussian variables $X$ and $Y$, the solution to the IB optimization problem is a noisy linear transformation $T = A_\beta X + \eta$, where $\eta$ is Gaussian noise. Moreover, there exist critical values $\beta_n^c$ such that $0 \leq \beta_i^c \leq \beta_j^c$ whenever $i < j$, and
\begin{equation}
    \rank(A_\beta) = n \quad \text{for} \quad \beta \in (\beta_n^c, \beta_{n+1}^c).
\end{equation}
\end{theorem}

This theorem indicates that as we adjust the trade-off parameter $\beta$, there are bifurcation points where the rank of the optimal linear transformation $A_\beta$ changes. This corresponds to changes in the dimensionality of the representation $T$, which is directly related to the local rank in the case of neural networks.

\subsection{Empirical Evidence Connecting Local Rank and Information Compression}

After establishing an analytical connection between local rank and the trade-off parameter $\beta$ in the Gaussian case, we now empirically test this relationship on both synthetic and more complex datasets. We train Deep Variational Information Bottleneck (VIB) models \citep{alemi2016deep} and observe the effect of the IB trade-off parameter $\beta$ and analytical phase transitions on the local rank of the encoder.

\begin{figure}[ht]
    \centering
    \includegraphics[width=\linewidth]{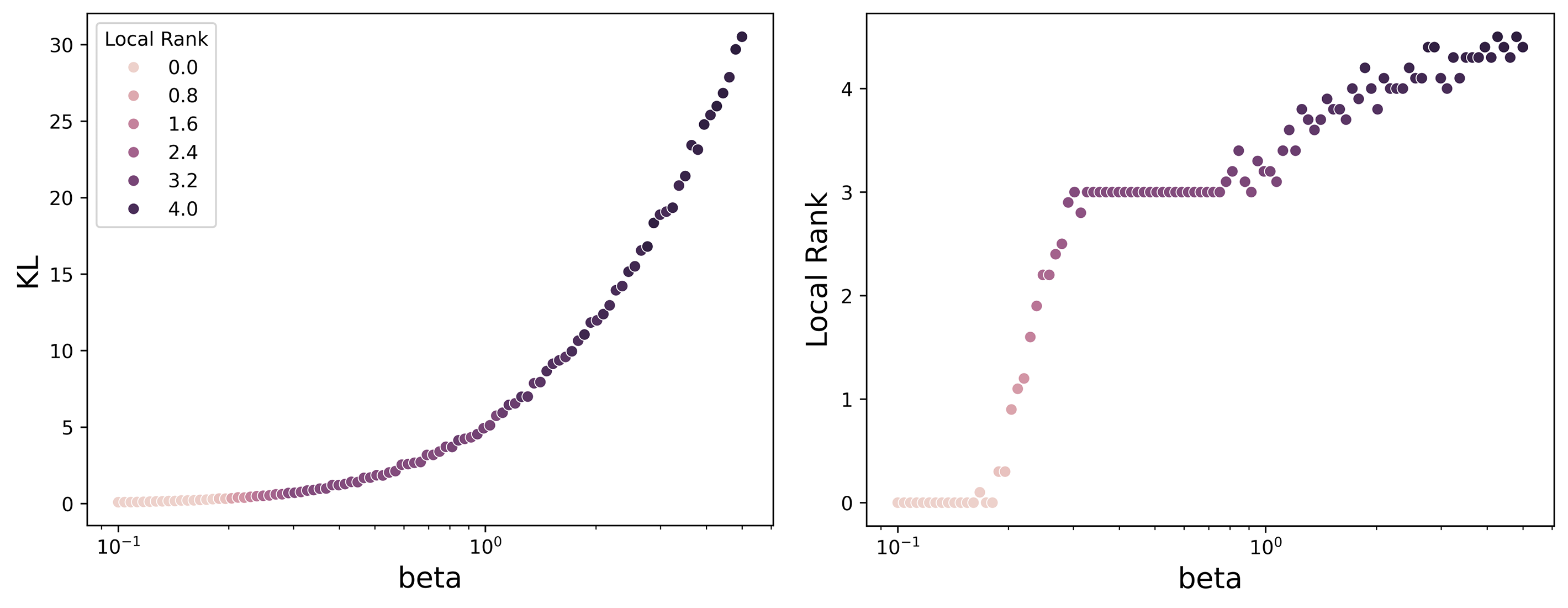}
    \caption{\textbf{Empirical Results on Gaussian Data using Deep-VIB.}
    \textbf{Left:} KL divergence component of the loss versus $\beta$, with points colored by empirical local rank corresponding to critical $\beta$ values.
    \textbf{Right:} Local rank as a function of $\beta$, showing an increase with $\beta$ and distinct phase transitions. We provide more information in Appendix \ref{sec:Gauss_DEEPCIB}.} 
    \label{fig:gaussian_vib}
\end{figure}

\paragraph{Gaussian Data.} In the first experiment, we train Deep VIB models to map between two correlated Gaussian distributions in $\R^5$. In the \textbf{left} of Figure \ref{fig:gaussian_vib}, we show that the KL divergence component of the loss varies with $\beta$, as expected. The points are colored according to the empirical local rank values, which correspond to the closest critical $\beta$ values predicted by the theory. In the \textbf{right} of the figure, we plot the local rank as a function of $\beta$, observing that it increases with $\beta$ and there is a distinct phase transition, aligning with the theoretical predictions. See Appendix \ref{sec:Gauss_DEEPCIB} for more information.

\paragraph{MNIST and Fashion-MNIST Data.} In the second experiment, we train Deep VIB models on the MNIST \citep{mnist} and Fashion-MNIST \citep{xiao2017online} datasets. As we increase $\beta$, we observe that the local rank increases and the accuracy decreases.

\begin{figure}[ht]
    \centering
    \includegraphics[width=\linewidth]{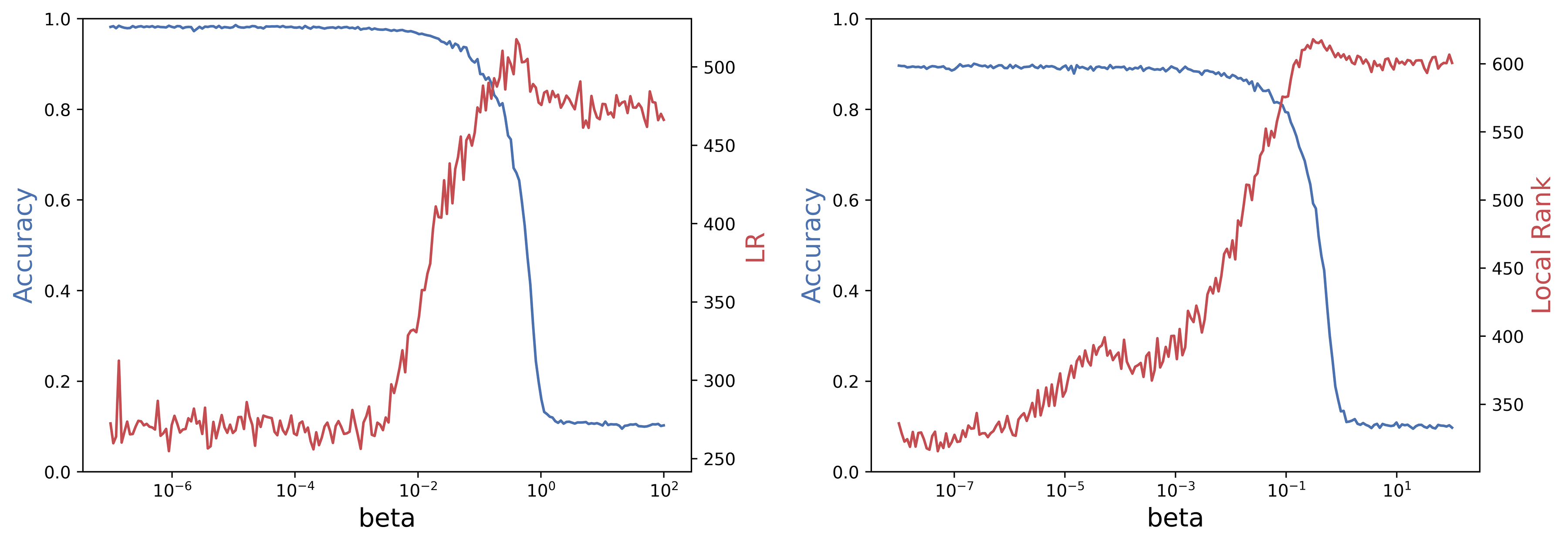}
    \caption{Empirical Results on MNIST and Fashion-MNIST.
    \textbf{Left:} MNIST dataset.
    \textbf{Right:} Fashion-MNIST dataset. As we increase $\beta$, the local rank increases, and accuracy decreases, indicating that higher $\beta$ values correspond to less compressed representations and lower performance.}
    \label{fig:vib_local_rank}
\end{figure}

In both experiments, changing $\beta$ leads to a reduction in local rank, indicating increased information compression. This supports our hypothesis that local rank is indicative of the level of information compression in the network, even for complex datasets.

\section{Discussion and Future Work}
\label{sec:discussion}

Our work introduces the concept of local rank as a meaningful measure of the dimensionality of feature manifolds in deep neural networks. We have demonstrated both theoretically and empirically that local rank decreases during the terminal phase of training, suggesting that networks compress the dimensionality of their representations.

By connecting local rank to the Information Bottleneck theory, we provide a new perspective on how neural networks manage the trade-off between compression and prediction. Our findings imply that networks naturally perform information compression by reducing the rank of their learned representations.

Understanding the behavior of local rank has implications for model compression, generalization, and the design of neural network architectures. Future research could formalize the relationship between local rank and mutual information in non-Gaussian settings, extend the analysis to other network architectures, and explore practical applications in compression techniques.

\newpage


\medskip

\small

\bibliography{references}

\begin{thebibliography}{28}
\providecommand{\natexlab}[1]{#1}
\providecommand{\url}[1]{\texttt{#1}}
\expandafter\ifx\csname urlstyle\endcsname\relax
  \providecommand{\doi}[1]{doi: #1}\else
  \providecommand{\doi}{doi: \begingroup \urlstyle{rm}\Url}\fi

\bibitem[Alemi et~al.(2016)Alemi, Fischer, Dillon, and Murphy]{alemi2016deep}
Alexander~A Alemi, Ian Fischer, Joshua~V Dillon, and Kevin Murphy.
\newblock Deep variational information bottleneck.
\newblock \emph{arXiv preprint arXiv:1612.00410}, 2016.

\bibitem[Ansuini et~al.(2019)Ansuini, Laio, Macke, and Zoccolan]{ansuini2019intrinsic}
Alessio Ansuini, Alessandro Laio, Jakob~H Macke, and Davide Zoccolan.
\newblock Intrinsic dimension of data representations in deep neural networks.
\newblock \emph{Advances in Neural Information Processing Systems}, 32, 2019.

\bibitem[Arora et~al.(2019)Arora, Cohen, Hu, and Luo]{arora2019implicit}
Sanjeev Arora, Nadav Cohen, Wei Hu, and Yuping Luo.
\newblock Implicit regularization in deep matrix factorization.
\newblock \emph{Advances in Neural Information Processing Systems}, 32, 2019.

\bibitem[Ben-Shaul et~al.(2023)Ben-Shaul, Shwartz-Ziv, Galanti, Dekel, and LeCun]{ben2023reverse}
Ido Ben-Shaul, Ravid Shwartz-Ziv, Tomer Galanti, Shai Dekel, and Yann LeCun.
\newblock Reverse engineering self-supervised learning.
\newblock \emph{Advances in Neural Information Processing Systems}, 36:\penalty0 58324--58345, 2023.

\bibitem[Chechik et~al.(2003)Chechik, Globerson, Tishby, and Weiss]{NIPS2003_7e05d6f8}
Gal Chechik, Amir Globerson, Naftali Tishby, and Yair Weiss.
\newblock Information bottleneck for gaussian variables.
\newblock In S.~Thrun, L.~Saul, and B.~Sch\"{o}lkopf (eds.), \emph{Advances in Neural Information Processing Systems}, volume~16. MIT Press, 2003.
\newblock URL \url{https://proceedings.neurips.cc/paper_files/paper/2003/file/7e05d6f828574fbc975a896b25bb011e-Paper.pdf}.

\bibitem[Feng et~al.(2022)Feng, Zheng, Huang, Zhao, Jordan, and Zha]{feng2022rank}
Ruili Feng, Kecheng Zheng, Yukun Huang, Deli Zhao, Michael Jordan, and Zheng-Jun Zha.
\newblock Rank diminishing in deep neural networks.
\newblock \emph{Advances in Neural Information Processing Systems}, 35:\penalty0 33054--33065, 2022.

\bibitem[Gunasekar et~al.(2017)Gunasekar, Woodworth, Bhojanapalli, Neyshabur, and Srebro]{gunasekar2017implicit}
Suriya Gunasekar, Blake~E Woodworth, Srinadh Bhojanapalli, Behnam Neyshabur, and Nati Srebro.
\newblock Implicit regularization in matrix factorization.
\newblock \emph{Advances in neural information processing systems}, 30, 2017.

\bibitem[Humayun et~al.(2024)Humayun, Amara, Schumann, Farnadi, Rostamzadeh, and Havaei]{humayun2024understanding}
Ahmed~Imtiaz Humayun, Ibtihel Amara, Candice Schumann, Golnoosh Farnadi, Negar Rostamzadeh, and Mohammad Havaei.
\newblock Understanding the local geometry of generative model manifolds.
\newblock \emph{arXiv preprint arXiv:2408.08307}, 2024.

\bibitem[Jacot(2023{\natexlab{a}})]{jacot2}
Arthur Jacot.
\newblock Bottleneck structure in learned features: Low-dimension vs regularity tradeoff.
\newblock In A.~Oh, T.~Naumann, A.~Globerson, K.~Saenko, M.~Hardt, and S.~Levine (eds.), \emph{Advances in Neural Information Processing Systems}, volume~36, pp.\  23607--23629. Curran Associates, Inc., 2023{\natexlab{a}}.
\newblock URL \url{https://proceedings.neurips.cc/paper_files/paper/2023/file/4a6695df88f2de0d49f875189ea181ef-Paper-Conference.pdf}.

\bibitem[Jacot(2023{\natexlab{b}})]{jacot2023implicit}
Arthur Jacot.
\newblock Implicit bias of large depth networks: a notion of rank for nonlinear functions.
\newblock In \emph{The Eleventh International Conference on Learning Representations}, 2023{\natexlab{b}}.

\bibitem[Ji \& Telgarsky(2020)Ji and Telgarsky]{NEURIPS2020_c76e4b2f}
Ziwei Ji and Matus Telgarsky.
\newblock Directional convergence and alignment in deep learning.
\newblock In H.~Larochelle, M.~Ranzato, R.~Hadsell, M.F. Balcan, and H.~Lin (eds.), \emph{Advances in Neural Information Processing Systems}, volume~33, pp.\  17176--17186. Curran Associates, Inc., 2020.
\newblock URL \url{https://proceedings.neurips.cc/paper_files/paper/2020/file/c76e4b2fa54f8506719a5c0dc14c2eb9-Paper.pdf}.

\bibitem[Lecun et~al.(1998)Lecun, Bottou, Bengio, and Haffner]{mnist}
Yann Lecun, Leon Bottou, Y.~Bengio, and Patrick Haffner.
\newblock Gradient-based learning applied to document recognition.
\newblock \emph{Proceedings of the IEEE}, 86:\penalty0 2278 -- 2324, 12 1998.
\newblock \doi{10.1109/5.726791}.

\bibitem[Lyu \& Li(2020)Lyu and Li]{Lyu2020Gradient}
Kaifeng Lyu and Jian Li.
\newblock Gradient descent maximizes the margin of homogeneous neural networks.
\newblock In \emph{International Conference on Learning Representations}, 2020.
\newblock URL \url{https://openreview.net/forum?id=SJeLIgBKPS}.

\bibitem[Neyshabur et~al.(2014)Neyshabur, Tomioka, and Srebro]{neyshabur2014search}
Behnam Neyshabur, Ryota Tomioka, and Nathan Srebro.
\newblock In search of the real inductive bias: On the role of implicit regularization in deep learning.
\newblock \emph{arXiv preprint arXiv:1412.6614}, 2014.

\bibitem[Olah(2014)]{olah2014neural}
Christopher Olah.
\newblock Neural networks, manifolds, and topology.
\newblock \emph{Blog post}, 2014.

\bibitem[Papyan et~al.(2020)Papyan, Han, and Donoho]{papyan2020prevalence}
Vardan Papyan, X.~Y. Han, and David~L. Donoho.
\newblock Prevalence of neural collapse during the terminal phase of deep learning training.
\newblock \emph{Proceedings of the National Academy of Sciences}, 117\penalty0 (40):\penalty0 24652--24663, 2020.
\newblock \doi{10.1073/pnas.2015509117}.
\newblock URL \url{https://www.pnas.org/doi/abs/10.1073/pnas.2015509117}.

\bibitem[Saxe et~al.(2018)Saxe, Bansal, Dapello, Advani, Kolchinsky, Tracey, and Cox]{michael2018on}
Andrew~Michael Saxe, Yamini Bansal, Joel Dapello, Madhu Advani, Artemy Kolchinsky, Brendan~Daniel Tracey, and David~Daniel Cox.
\newblock On the information bottleneck theory of deep learning.
\newblock In \emph{International Conference on Learning Representations}, 2018.
\newblock URL \url{https://openreview.net/forum?id=ry_WPG-A-}.

\bibitem[Shwartz-Ziv(2022)]{shwartz2022information}
Ravid Shwartz-Ziv.
\newblock Information flow in deep neural networks.
\newblock \emph{arXiv preprint arXiv:2202.06749}, 2022.

\bibitem[Shwartz-Ziv \& Tishby(2017)Shwartz-Ziv and Tishby]{shwartz2017opening}
Ravid Shwartz-Ziv and Naftali Tishby.
\newblock Opening the black box of deep neural networks via information.
\newblock \emph{arXiv preprint arXiv:1703.00810}, 2017.

\bibitem[Shwartz-Ziv et~al.(2019)Shwartz-Ziv, Painsky, and Tishby]{shwartzrepresentation}
Ravid Shwartz-Ziv, Amichai Painsky, and Naftali Tishby.
\newblock Representation compression and generalization in deep neural networks, 2019.
\newblock In \emph{URL https://openreview. net/forum}, 2019.

\bibitem[Shwartz-Ziv et~al.(2023)Shwartz-Ziv, Balestriero, Kawaguchi, Rudner, and LeCun]{NEURIPS2023_6b1d4c03}
Ravid Shwartz-Ziv, Randall Balestriero, Kenji Kawaguchi, Tim G.~J. Rudner, and Yann LeCun.
\newblock An information theory perspective on variance-invariance-covariance regularization.
\newblock In A.~Oh, T.~Naumann, A.~Globerson, K.~Saenko, M.~Hardt, and S.~Levine (eds.), \emph{Advances in Neural Information Processing Systems}, volume~36, pp.\  33965--33998. Curran Associates, Inc., 2023.
\newblock URL \url{https://proceedings.neurips.cc/paper_files/paper/2023/file/6b1d4c03391b0aa6ddde0b807a78c950-Paper-Conference.pdf}.

\bibitem[Súkeník et~al.(2024)Súkeník, Mondelli, and Lampert]{súkeník2024neuralcollapseversuslowrank}
Peter Súkeník, Marco Mondelli, and Christoph Lampert.
\newblock Neural collapse versus low-rank bias: Is deep neural collapse really optimal?, 2024.
\newblock URL \url{https://arxiv.org/abs/2405.14468}.

\bibitem[Timor et~al.(2023{\natexlab{a}})Timor, Vardi, and Shamir]{pmlr-v201-timor23a}
Nadav Timor, Gal Vardi, and Ohad Shamir.
\newblock Implicit regularization towards rank minimization in relu networks.
\newblock In Shipra Agrawal and Francesco Orabona (eds.), \emph{Proceedings of The 34th International Conference on Algorithmic Learning Theory}, volume 201 of \emph{Proceedings of Machine Learning Research}, pp.\  1429--1459. PMLR, 20 Feb--23 Feb 2023{\natexlab{a}}.
\newblock URL \url{https://proceedings.mlr.press/v201/timor23a.html}.

\bibitem[Timor et~al.(2023{\natexlab{b}})Timor, Vardi, and Shamir]{timor2023implicit}
Nadav Timor, Gal Vardi, and Ohad Shamir.
\newblock Implicit regularization towards rank minimization in relu networks.
\newblock In \emph{International Conference on Algorithmic Learning Theory}, pp.\  1429--1459. PMLR, 2023{\natexlab{b}}.

\bibitem[Tishby et~al.(2000)Tishby, Pereira, and Bialek]{tishby2000information}
Naftali Tishby, Fernando~C Pereira, and William Bialek.
\newblock The information bottleneck method.
\newblock \emph{arXiv preprint physics/0004057}, 2000.

\bibitem[Wang et~al.(2021)Wang, Meng, Chen, and Liu]{wang2021implicit}
Bohan Wang, Qi~Meng, Wei Chen, and Tie-Yan Liu.
\newblock The implicit bias for adaptive optimization algorithms on homogeneous neural networks.
\newblock In \emph{International Conference on Machine Learning}, pp.\  10849--10858. PMLR, 2021.

\bibitem[Xiao et~al.(2017)Xiao, Rasul, and Vollgraf]{xiao2017online}
Han Xiao, Kashif Rasul, and Roland Vollgraf.
\newblock Fashion-mnist: a novel image dataset for benchmarking machine learning algorithms.
\newblock \emph{arXiv preprint arXiv:1708.07747}, 2017.

\bibitem[Zhang et~al.(2021)Zhang, Bengio, Hardt, Recht, and Vinyals]{zhang2021understanding}
Chiyuan Zhang, Samy Bengio, Moritz Hardt, Benjamin Recht, and Oriol Vinyals.
\newblock Understanding deep learning (still) requires rethinking generalization.
\newblock \emph{Communications of the ACM}, 64\penalty0 (3):\penalty0 107--115, 2021.

\end{thebibliography}


\appendix

\section{Proofs}
\label{sec:proofs}

\subsection{For Binary Classification}
\label{sec:proofsclass}
\begin{lemma}
    There exists $\epsilon_0 > 0$ such that for all $\epsilon < \epsilon_0$:
    \begin{equation}
        \rank_\epsilon(J_x p_l(x)) \leq \rank_\epsilon(W_l)
    \end{equation}
\end{lemma}

\begin{proof}
Notice first that a calculation gives us the following, where $W_i$ are the weight matrices and $D_i$ are diagonal matrices with $0$ or $1$ on the diagonal. These diagonal matrices correspond to the activation pattern of the ReLU functions at a specific layer. 

\[
J_x p_l(x) = W_l D_{l-1} W_{l-1} D_{l-2} \cdots W_0
\]

Notice now it is clear that:

\[
\rank (J_x p_l(x)) \leq \rank(W_l)
\]

Notice also that for any matrix $A$:  $\lim_{\epsilon \to 0} \rank_\epsilon(A) = \rank(A)$. Since $\rank_\epsilon(\cdot)$ takes on values in a discrete set, its clear that there exists some $\epsilon_0 > 0$ such that $\epsilon < \epsilon_0$ implies that:

\[
\rank_\epsilon (J_x p_l(x)) \leq \rank_\epsilon(W_l)
\]

\end{proof}

We now recall a theorem courtesy of \cite{pmlr-v201-timor23a}:

\begin{theorem}
\label{opt_thm}
\textbf{(Quoted from \cite{pmlr-v201-timor23a})} Let $\{(x_i, y_i)\}_{i=1}^n \subseteq \mathbb{R}^{n_0} \times \{-1,1\}$ be a binary classification dataset, and assume that there is $i \in [n]$ with $\|x_i\| \leq 1$. Assume that there is a fully-connected neural network $N$ of width $m \geq 2$ and depth $k \geq 2$, such that for all $i \in [n]$ we have $y_i N(x_i) \geq 1$, and the weight matrices $W_1, \ldots, W_k$ of $N$ satisfy $\|W_i\|_F \leq B$ for some $B > 0$. Let $N_\theta$ be a fully-connected neural network of width $m' \geq m$ and depth $k' > k$ parameterized by $\theta$. Let $\theta^* = [W_1^*, \ldots, W_{L}^*]$ be a global optimum of the above optimization problem \ref{optimization_problem}. Namely, $\theta^*$ parameterizes a minimum-norm fully-connected network of width $n_l$ and depth $L$ that labels the dataset correctly with margin 1. Then, we have
\begin{equation}
\frac{1}{L} \sum_{i=1}^{L} \frac{\|W_i^*\|_\sigma}{\|W_i^*\|_F} \geq \frac{1}{\sqrt{2}} \cdot \left(\frac{\sqrt{2}}{B}\right)^{\frac{k}{L}} \cdot \sqrt{\frac{L}{L + 1}}.
\end{equation}
Equivalently, we have the following upper bound on the harmonic mean of the ratios $\frac{\|W_i^*\|_F}{\|W_i^*\|_\sigma}$:
\begin{equation}
\frac{L}{\sum_{i=1}^{L} \left(\frac{\|W_i^*\|_F}{\|W_i^*\|_\sigma}\right)^{-1}} 
\leq
\sqrt{2} \cdot \left(\frac{B}{\sqrt{2}}\right)^{\frac{k}{L}} \cdot \sqrt{\frac{L + 1}{L}}.
\end{equation}
\end{theorem}

For convenience we restate our proposition in full formality:

\begin{proposition}
Let $\{(x_i, y_i)\}_{i=1}^n \subseteq \mathbb{R}^{n_0} \times \{-1,1\}$ be a binary classification dataset, and assume that there is $i \in [n]$ with $\|x_i\| \leq 1$. Assume that there is a fully-connected neural network $N$ of width $m \geq 2$ and depth $k \geq 2$, such that for all $i \in [n]$ we have $y_i N(x_i) \geq 1$, and the weight matrices $W_1, \ldots, W_k$ of $N$ satisfy $\|W_i\|_F \leq B$ for some $B > 0$. Let $N_\theta$ be a fully-connected neural network of width $m' \geq m$ and depth $k' > k$ parameterized by $\theta$. Let $\theta^* = [W_1^*, \ldots, W_{L}^*]$ be a global optimum of the above optimization problem \ref{optimization_problem}. Namely, $\theta^*$ parameterizes a minimum-norm fully-connected network of width $n_l$ and depth $L$ that labels the dataset correctly with margin 1. Then, there exists an $l\in \{1, \cdots , L\}$ and $\epsilon_0 > 0$ such that for $0 < \epsilon <\epsilon_0$ the following holds:

\begin{equation}
    \frac{\textbf{LR}_l^\epsilon}{||W_l^*||_\sigma^2} \leq \frac{2}{\epsilon^2} \cdot (\frac{B}{\sqrt{2}})^{\frac{2k}{L}} \cdot \frac{L+1}L
\end{equation}
\end{proposition}


\begin{proof}
Notice first that for any arbitrary matrix $A$, we have that,

\[
||A||_F = \sqrt{\sum_{i=1}^n \sigma_i^2} \geq \sqrt{\rank_\epsilon(A)\epsilon^2} = \epsilon \sqrt{\rank_\epsilon(A)}
\]

So then,

\begin{equation}
    \frac{||A||_F}{||A||_\sigma} \geq \frac{\epsilon}{||A||_\sigma} \sqrt{\rank_\epsilon(A)}
\end{equation}

Notice now that from application of the theorem in \cite{pmlr-v201-timor23a}, we can get that harmonic mean of the quantities $\frac{||W_i^*||_F}{||W_i^*||_\sigma}$ is bounded. In particular this means that there exists at least one $l \in \{1, \cdots, L\}$ such that $\frac{||W_l^*||_F}{||W_l^*||_\sigma}$ lies below this harmonic mean. So then,

\begin{align}
   \sqrt{2} \cdot \left(\frac{B}{\sqrt{2}}\right)^{\frac{k}{L}} \cdot \sqrt{\frac{L + 1}{L}}
&\geq 
\frac{||W_l^*||_F}{||W_l^*||_\sigma}   \\
&\geq \frac{\epsilon}{||W_l^*||_\sigma} \sqrt{\rank_\epsilon(W_l^*)}
\end{align}

Re-arranging terms and squaring both sides gives us:

\[
\frac{\rank_\epsilon(W_l^*)}{||W_l^*||_\sigma^2} \leq \frac{2}{\epsilon^2} \cdot (\frac{B}{\sqrt{2}})^{\frac{2k}{L}} \cdot \frac{L+1}L
\]

Now if we can apply lemma 1 and we get that there exists $\epsilon' > 0$ such that for any $0 < \epsilon < \epsilon'$:

\[
\rank_\epsilon(J_x p_l(x)) \leq \rank_\epsilon(W_l^*)
\]

So then, we get:
\begin{equation}
    \frac{\rank_\epsilon(J_x p_l(x))}{||W_l^*||_\sigma^2} \leq \frac{2}{\epsilon^2} \cdot (\frac{B}{\sqrt{2}})^{\frac{2k}{L}} \cdot \frac{L+1}L
\end{equation}

Since this holds for any $x$, we can then take expectation with respect to the data distribution on the left hand side and we get that:

\begin{equation}
    \frac{\textbf{LR}_l^\epsilon}{||W_l^*||_\sigma^2} \leq \frac{2}{\epsilon^2} \cdot (\frac{B}{\sqrt{2}})^{\frac{2k}{L}} \cdot \frac{L+1}L
\end{equation}

\end{proof}

\subsection{For Interpolating Neural Networks}
\label{sec:proofsinterp}
As before we recall the analogous theorem from \cite{timor2023implicit} for interpolating neural networks.

\begin{theorem}
\label{opt_thm1}
\textbf{(Quoted from \cite{pmlr-v201-timor23a})}
Let $\{(x_i, y_i)\}_{i=1}^n \subset \mathbb{R}^{n_0} \times \mathbb{R}_+$ be a dataset, and assume that there is $i \in [n]$ with $\|x_i\| \leq 1$ and $y_i \geq 1$. Assume that there is a fully-connected neural network $N$ of width $m \geq 2$ and depth $k \geq 2$, such that for all $i \in [n]$ we have $N(x_i) = y_i$, and the weight matrices $W_1, \dots, W_k$ of $N$ satisfy $\|W_i\|_F \leq B$ for some $B > 0$. Let $N_\theta$ be a fully-connected neural network of width $m' \geq m$ and depth $L > k$ parameterized by $\theta$. Let $\theta^* = [W_1^*, \dots, W_{L}^*]$ be a global optimum of the following problem:

\begin{equation}
\min_\theta \|\theta\| \quad \text{s.t.} \quad \forall i \in [n] \ N_\theta(x_i) = y_i.
\end{equation}

Then,
\begin{equation}
\frac{1}{L} \sum_{i=1}^{L} \frac{\|W_i^*\|_\sigma}{\|W_i^*\|_F} \geq \left(\frac{1}{B}\right)^{\frac{k}{L}}.
\end{equation}

Equivalently, we have the following upper bound on the harmonic mean of the ratios $\frac{\|W_i^*\|_F}{\|W_i^*\|_\sigma}$:
\begin{equation}
\frac{L}{\sum_{i=1}^{L} \left(\frac{\|W_i^*\|_F}{\|W_i^*\|_\sigma}\right)^{-1}} \leq B^{\frac{k}{L}}.
\end{equation}

\end{theorem}

We can now restate our proposition with a proof.

\begin{proposition}
    Let $\{(x_i, y_i)\}_{i=1}^n \subset \mathbb{R}^{n_0} \times \mathbb{R}_+$ be a dataset, and assume that there is $i \in [n]$ with $\|x_i\| \leq 1$ and $y_i \geq 1$. Assume that there is a fully-connected neural network $N$ of width $m \geq 2$ and depth $k \geq 2$, such that for all $i \in [n]$ we have $N(x_i) = y_i$, and the weight matrices $W_1, \dots, W_k$ of $N$ satisfy $\|W_i\|_F \leq B$ for some $B > 0$. Let $N_\theta$ be a fully-connected neural network of width $m' \geq m$ and depth $L > k$ parameterized by $\theta$. Let $\theta^* = [W_1^*, \dots, W_{L}^*]$ be a global optimum of the following problem:

\begin{equation}
\min_\theta \|\theta\| \quad \text{s.t.} \quad \forall i \in [n], \ \N_\theta(x_i) = y_i.
\end{equation}

Then, there exist an $l\in \{1, \cdots , L\}$ and $\epsilon_0 > 0$ such that for $0 < \epsilon <\epsilon_0$ the following holds:

\begin{equation}
    \frac{\textbf{LR}_l^\epsilon}{||W_l||_\sigma^2}   \leq \frac{B^\frac{2k}{L}}{\epsilon^2}
\end{equation}
\end{proposition}

\begin{proof}
We first apply the prior theorem to get that the harmonic mean of the ratios of the Frobenius norm to the operator norm of the weight matrices are bounded like:

\begin{equation}
\frac{L}{\sum_{i=1}^{L} \left(\frac{\|W_i^*\|_F}{\|W_i^*\|_\sigma}\right)^{-1}} \leq B^{\frac{k}{L}}.
\end{equation}

In particular, there exists some layer $l$ such that its ratio falls below the Harmonic mean, so then:

\begin{equation}
    \frac{\|W_i^*\|_F}{\|W_i^*\|_\sigma} \leq B^{\frac{k}{L}}.
\end{equation}

Now recall that for any matrix $A$ we have:
\begin{equation}
    \frac{||A||_F}{||A||_\sigma} \geq \frac{\epsilon}{||A||_\sigma} \sqrt{\rank_\epsilon(A)}.
\end{equation}

Now apply this to $W_l$ and we get that:
\begin{equation}
        \frac{||W_l||_F}{||W_l||_\sigma} \geq \frac{\epsilon}{||W_l||_\sigma} \sqrt{\rank_\epsilon(W_l)}.
\end{equation}

We can now apply the lemma and we get that there exists $\epsilon_0 >0$ such that for any $0<\epsilon < \epsilon_0$:

\begin{equation}
        \frac{||W_l||_F}{||W_l||_\sigma} \geq \frac{\epsilon}{||W_l||_\sigma} \sqrt{\rank_\epsilon(W_l)} \geq \frac{\epsilon}{||W_l||_\sigma} \sqrt{\rank_\epsilon(J_x p_l (x))} .
\end{equation}

So then it follows that:

\begin{equation}
    \frac{\epsilon}{||W_l||_\sigma} \sqrt{\rank_\epsilon(J_x p_l (x))}  \leq B^\frac{k}{L}.
\end{equation}

Or equivalently, 

\begin{equation}
    \frac{\epsilon^2}{||W_l||_\sigma^2} \rank_\epsilon(J_x p_l (x))  \leq B^\frac{2k}{L}.
\end{equation}

Taking expectation over $x \sim \textbf{Data}$ now completes the proof as:

\begin{equation}
    \frac{\textbf{LR}_l^\epsilon}{||W_l||_\sigma^2}   \leq \frac{B^\frac{2k}{L}}{\epsilon^2}.
\end{equation}

\end{proof}

\section{On the Gaussian Deep VIB}
\label{sec:Gauss_DEEPCIB}
For figure \ref{fig:gaussian_vib}, our Deep VIB model is trained to map $X$ to $Y$ using a Deep Linear network as the encoder. Here we take both of these to be isotropic Gaussians in $\R^5$. We use the following cross-covariance matrix:

\begin{equation}
    \Sigma_{XY} = \begin{pmatrix}
0.1 & 0   & 0   & 0   & 0 \\
0   & 0.1 & 0   & 0   & 0 \\
0   & 0   & 0.5 & 0   & 0 \\
0   & 0   & 0   & 0.5 & 0 \\
0   & 0   & 0   & 0   & 0.5
\end{pmatrix}.
\end{equation}

For an intuition, this means that the last 3 variables are highly correlated between $X$ and $Y$, whereas the first two variables are only somewhat correlated. The theory would then suggest a phase transition, and a distinct jump from a rank to a rank of $3$ as we lower $\beta$. We note that we can observe this structure in figure \ref{fig:gaussian_vib} (right). 

\end{document}